\documentclass{article}
\usepackage{amsmath, amsthm, amsfonts, amssymb, enumerate}
\usepackage[pdftex]{graphicx}
\usepackage{samepaper, samemath, samepage}
\usepackage{cite}
\newcommand{\eps}{\epsilon}
\newcommand{\vp}{\vec p}

\newcommand{\X}{\mathcal X}

\renewcommand{\Z}{\mathcal Z}
\renewcommand{\P}{\mathcal P}
\renewcommand{\F}{\mathcal F}

\DeclareMathOperator{\Beta}{Beta}
\DeclareMathOperator{\Dir}{Dir}

\title{Bayesian Adaptive Data Analysis Guarantees from Subgaussianity}
\author{Sam Elder\\MIT}

\begin{document}

\maketitle

\begin{abstract}
The new field of adaptive data analysis seeks to provide algorithms and provable guarantees for models of machine learning that allow researchers to reuse their data, which normally falls outside of the usual statistical paradigm of static data analysis. In 2014, Dwork, Feldman, Hardt, Pitassi, Reingold and Roth \cite{dwork2015preserving} introduced one potential model and proposed several solutions based on differential privacy. In previous work in 2016 \cite{elder2016challenges}, we described a problem with this model and instead proposed a Bayesian variant, but also found that the analogous Bayesian methods cannot achieve the same statistical guarantees as in the static case.

In this paper, we prove the first positive results for the Bayesian model, showing that with a Dirichlet prior, the posterior mean algorithm indeed matches the statistical guarantees of the static case. The main ingredient is a new theorem showing that the $\Beta(\alpha,\beta)$ distribution is subgaussian with variance proxy $O(1/(\alpha+\beta+1))$, a concentration result also of independent interest. We provide two proofs of this result: a probabilistic proof utilizing a simple condition for the raw moments of a positive random variable and a learning-theoretic proof based on considering the beta distribution as a posterior, both of which have implications to other related problems.
\end{abstract}

\section{Introduction}

The field of adaptive data analysis is motivated by the common practice in machine learning of \emph{data reuse}. The potential problems with this were perhaps first illustrated by Freedman in 1983 \cite{freedman1983note}, who showed that using the same data to select regressors and subsequently fit a general linear model to those regressors would lead to wildly inaccurate estimates of the goodness of fit, at least when the number of regressors and data points were on the same order of magnitude.

The simplest solution to such problems is to collect a fresh set of data each time the overall model is updated, an approach known as \emph{sample splitting}. But in many scenarios, this is a rather expensive solution, especially for data practices that utilize multiple such adaptive rounds. Stated quantitatively, this requires a sample complexity that scales linearly with the number of measurements of the data.

The aim of adaptive data analysis, as introduced by Dwork, Feldman, Hardt, Pitassi, Reignold and Roth \cite{dwork2015preserving} is to reduce this dependence while preserving standard statistical guarantees by introducing a protective layer of interaction between the analyst and the data. To describe this challenge in a mathematical framework, they introduced a two-player game between an analyst and a curator.

In the original model, the analyst is seeking to answer some number $q$ of \emph{statistical queries} about a distribution $\vec p$ on a universe $\X$, which are expressed as the averages of bounded functions $f:\X\to[0,1]$ of the data. For instance, if $f:\X\to\{0,1\}$, this is asking the curator to estimate a probability, known as a \emph{counting query}. The curator receives $n$ independent samples from $\vec p$ and is seeking to answer every query to within an additive error $\epsilon$ on the population the data is drawn from, with probability $1-\delta$. The central question is how many samples $n=n(q,\epsilon,\delta)$ the curator requires to achieve this accuracy.

If the queries are specified in one batch in advance of the analyst receiving the answers, this is known as static data analysis, and $n=\Theta\left(\frac1{\eps^2}\log\frac q\delta\right)$ samples are both necessary (on certain problems) and sufficient. In this case, the curator can simply answer with the empirical means of the functions: $\frac1n\sum_{i=1}^nf(x_i)$, where $x_1,\dotsc,x_n$ are the data points he receives. The full framework of adaptive data analysis, where later queries are allowed to depend on previous answers, is much less well-understood and the subject of this line of investigation.

As we previously argued in \cite{elder2016challenges}, though, the hardest problems introduced in this framework tend not to reflect the reality it is attempting to model. In fact, since the analyst is taken to be adversarial and worst-case, we could suppose that the analyst already knows the distribution $\vec p$! To avoid such scenarios and hopefully better match reality, we introduced a Bayesian variant of the problem, where we enforce that the curator and analyst have the same information via a common accurate Bayesian prior $\P$ on $\vec p$.

The first question, which we addressed in \cite{elder2016challenges}, is what sorts of adaptive analyst strategies could cause problems for the curator without utilizing side information. The main result in that work was a problem on which a large class of curator algorithms would fail. This problem had two components: A difficult problem based on error-correcting codes that would produce high posterior uncertainty in some direction and an augmentation technique to allow the analyst to extract information from these curator algorithms using slightly correlated queries.

Quantitatively, as previously discussed, sample splitting methods can trivially achieve a linear dependence: $n\propto q$. The original dimension-free work of DFHPRR improved this to $n\propto\sqrt q$, but this is far short of the static bound cited above, $n\propto\log q$. In \cite{elder2016challenges}, we ultimately demonstrated a high-dimensional problem on which a large class of curator algorithms require about $n\propto\sqrt[4]q$.

In this paper, we demonstrate the first positive results in Bayesian adaptive data analysis. In particular, we show that if the common prior is a Dirichlet distribution, the most natural Bayesian curator algorithm, the posterior mean, achieves the same level of accuracy against adaptive queries as static queries: $n=O\left(\frac1{\eps^2}\log\frac q\delta\right)$.

This positive result is also rather different than previous positive claims. All previously proposed techniques in this area, from the differential privacy-based techniques introduced in DFHPRR \cite{dwork2015preserving}, to follow-up work restricting to the special case of a machine learning competition leaderboard by Blum and Hardt \cite{blum2015ladder}, to the wide class of techniques analyzed in \cite{elder2016challenges}, focus on attempting to obfuscate the answers in order to prevent the analyst from overfitting to the data. By contrast, we show in this paper that under a Dirichlet prior, the curator can achieve the static guarantee even without any obfuscation!

To prove such a result, the curator must be confident in answers to any possible query. Our main tool for proving such confidence will be the probabilistic notion of subgaussianity. In section \ref{sqa}, we recall the definition and key properties of this strong notion of concentration. We then prove the main technical result in section \ref{bpbt} in two ways, showing that the $\Beta(\alpha,\beta)$ distribution is $O\left(\frac1{\alpha+\beta+1}\right)$-subgaussian. The first method consists of bounding the moment generating function in terms of the raw moments of the distribution, which are easy to compute. The second method is simpler and stronger, but more mysteriously requires one to consider the beta distribution as a Bayesian posterior. In section \ref{dpcrv}, we show how this implies that the curator is correct on all queries given a Dirichlet prior.

We then discuss extensions of these main results in section \ref{disc}. We first generalize the probabilistic approach on the beta distribution to a simple condition for (upper) subgaussianity on the raw moments of a nonnegative random variable. We also discuss several other conjugate priors, where empirical evidence suggests that the corresponding results are true. Then we consider the second learning-theoretic approach, and show that a series of simple conditions on the posterior mean evolution suffice, which helps to explain what makes the Dirichlet prior this friendly to adaptivity. Finally, in section \ref{conc}, we survey the implications for adaptive data analysis and offer directions for future work.

\section{Subgaussianity and Query Accuracy}\label{sqa}

A random variable $X$ with zero mean is said to be \emph{$\sigma^2$-subgaussian} if for all $\lambda$, $\E[\exp(\lambda X)]\le\exp(\lambda^2\sigma^2/2)$. Here, $\sigma^2$ is also known as the \emph{variance proxy}. Following the seminal work on subgaussian random variables \cite{buldygin2000metric}, define $\tau(X)=\min\{\sigma\ge0:X\text{ is }\sigma^2\text{-subgaussian}\}$. Recall the following:
\begin{prop}Some basic facts about subgaussian random variables.
\begin{enumerate}
\item As the name suggests, the variance gives a lower bound on the subgaussian variance proxy: $\Var[X]\le\tau^2(X)$. If the two are equal, $X$ is said to be \emph{strictly subgaussian}.
\item (See \cite{buldygin2000metric} Theorem 1.2) The space of subgaussian random variables is a Banach space with respect to the norm $\tau(X)$. That is, it has the right scaling, $\tau(aX)=a\tau(X)$ for all $a>0$, satisfies the triangle inequality $\tau(X+Y)\le\tau(X)+\tau(Y)$ and is complete on the space of subgaussian random variables.
\item (See \cite{buldygin2000metric} Theorem 1.3) $X$ is $\sigma^2$-subgaussian if for all integers $k\ge2$,
\[\E[X^{2k}]\le\left(\frac{\sigma^2}{\sqrt{3.1}}\right)^k(2k-1)!!,\qquad\text{where }(2k-1)!!=(2k-1)(2k-3)\dotsm(1)=\frac{(2k)!}{2^kk!}.\]
If $X$ is symmetric (i.e. $X$ and $-X$ have the same distributions, so all odd moments are zero), the factor of $\sqrt{3.1}$ can be dropped.
\end{enumerate}\label{subgaussianproperties}
\end{prop}
If $\E X\neq0$, we will abuse notation slightly and write $\tau(X)=\tau(X-\E X)$. That is, we will consider random variables $X$ with nonzero mean to be $\sigma^2$-subgaussian if their centered versions $X-\E X$ are. Note that in this context, conditions like in Proposition \ref{subgaussianproperties}.3 apply to the centered moments $\E[(X-\E X)^{2k}]$.

Now, let us return to the Bayesian adaptive data analysis problem and consider a single query for a moment, setting $q=1$. In this case, the static bound states that $n=O\left(\frac1{\eps^2}\log\frac1\delta\right)$ samples are sufficient for estimation of $\E_{x\sim\vp}[f(x)]$ to additive error $\epsilon$, with probability $1-\delta$. We will show that for this particular relationship between $\epsilon$ and $\delta$, this follows from a subgaussianity property on the posterior.

\begin{prop}If the curator's posterior distribution is $O(1/n)$-subgaussian with respect to every query, then the posterior mean-answering curator answers correctly and achieves the static sample complexity of $n=O\left(\frac1{\eps^2}\log\frac q\delta\right)$.\label{subgaussiansufficiency}\end{prop}
Before proving the proposition, let us explicate this condition. Each query $f:\X\to[0,1]$ projects every possible population $\vec p$ to a value $\E_{x\sim\vec p}f(x)\in[0,1]$. Therefore, $f$ projects every possible posterior distribution on populations to a distribution on $[0,1]$. The assumption is that this projected distribution of the curator's posterior is $O(1/n)$-subgaussian.

The claim is that as long as this condition holds, the curator doesn't care what queries the analyst asks; the posterior mean will be accurate on all of them. Indeed, the curator could actually release the entire posterior mean (not just its value on every query), giving the analyst everything she will ever learn from his answers. If we can show that the probability of error on any query is less than $\frac\delta q$, then a union bound will give a total error probability at most $\delta$, no matter which queries are asked.

\begin{proof}Suppose the curator's posterior distribution with respect to a given query is $c/n$-subgaussian. Equivalently stated, suppose the error $X=\E_{x\sim\vp}[f_i(x)]-a_i$ of the posterior mean is a (centered) $c/n$-subgaussian random variable. By Markov's inequality, setting $\lambda=n\eps/c$,
\[\Pr[X\ge\eps]\le\exp\left(-\lambda\eps\right)\E[\exp(\lambda X)]\le\exp\left(\frac{\lambda^2c}{2n}-\lambda\eps\right)=\exp\left(-\frac{n\eps^2}{2c}\right).\]
We can also prove an identical bound for $\Pr[X<-\eps]$, and therefore, the error probability $\delta_1$ satisfies
\[\delta_1\le2\exp\left(-\frac{n\eps^2}{2c}\right)\Leftrightarrow n\le\frac{2c}{\eps^2}\log\frac2{\delta_1}=O\left(\frac1{\eps^2}\log\frac1{\delta_1}\right).\]
Taking $\delta_1=\delta/q$, we have shown the required bound.\end{proof}

This is rather remarkable; it says that the curator can do just as well against adaptive queries as against static queries if his posterior has this concentration property. Of course, such a concentration property will not hold in every case; for instance, it is far from true for the posteriors in the series of examples considered in \cite{elder2016challenges}. But in cases where it does hold, like the Dirichlet prior and posterior we will investigate shortly, the curator doesn't have to do any obfuscation.

As an added bonus, the subgaussian framework also simplifies the set of queries we must consider:
\begin{prop}If the curator's posterior distribution is $O(1/n)$-subgaussian with respect to every \emph{counting} query, then the posterior mean curator answers correctly and achieves the static sample complexity of $n=O\left(\frac1{\eps^2}\log\frac q\delta\right)$.\label{countingsufficiency}\end{prop}
\begin{proof}The key here is Proposition \ref{subgaussianproperties}.2. Since $\tau(X)$ is a norm, convex combinations of $\sigma^2$-subgaussian random variables will also be $\sigma^2$-subgaussian:
\[\tau(\lambda X+(1-\lambda)Y)\le\lambda\tau(X)+(1-\lambda)\tau(Y)\le\lambda\sigma+(1-\lambda)\sigma=\sigma.\]
Recall that counting queries only have values $0$ or $1$, and therefore form the vertices in the hypercube of possible queries $[0,1]^\X$. All other queries are convex combinations of these\footnote{If $X$ is infinite, then some queries might not be convex combinations of finitely many counting queries, but we can find a sequence of finite convex combinations converging to them in $\infty$-norm simply by picking more possible values for the function in $[0,1]$ at each step. This will converge in $\infty$-norm, so the completeness of $\tau$ from Proposition \ref{subgaussianproperties}.2 shows that they must also have the same bound on their subgaussian norm.}, and therefore will be $O(1/n)$-subgaussian as well (with no loss in the constant). Therefore, by Proposition \ref{subgaussiansufficiency}, the posterior mean curator wins on all statistical queries if he wins on counting queries.\end{proof}

\section{Beta Priors on Bernoulli Trials}\label{bpbt}

The first posterior distribution we will investigate is the ubiquitous Beta distribution. We will see in the next section that this distribution is also the projection of the Dirichlet distribution on any counting query.

\begin{defn}The \emph{beta distribution} $\Beta(\alpha,\beta)$ is a continuous distribution on $[0,1]$ with density
\[\dfrac{\Gamma(\alpha+\beta)}{\Gamma(\alpha)\Gamma(\beta)}x^{\alpha-1}(1-x)^{\beta-1},\]
where $\Gamma(n)$ is the gamma function, satisfying $\Gamma(n)=(n-1)!$ for $n\in\N$.\end{defn}
The fraction in the density formula above is simply a normalization constant. Notice that for $\alpha=\beta=1$, this is the uniform distribution on $[0,1]$. For $\alpha=\beta\to0$, this converges to the Rademacher 0-1 random variable, $\frac12(\delta_{x,0}+\delta_{x,1})$. Finally, the (raw) moments of the beta distribution are given by
\begin{equation}\E X^k=\prod_{r=0}^{k-1}\frac{\alpha+r}{\alpha+\beta+r}=\frac{(\alpha)_k}{(\alpha+\beta)_k},\label{betamoments}\end{equation}
where $(x)_k=x(x+1)\dotsc(x+k-1)$ is a \emph{rising factorial}. In particular, the mean and variance are given by
\[\E[X]=\frac\alpha{\alpha+\beta}\quad\Var[X]=\frac{\alpha\beta}{(\alpha+\beta)^2(\alpha+\beta+1)}.\]

To see the beta distribution as a posterior, consider a Bernoulli trial, a single event with two possible outcomes: success or failure. If our prior over the probability of success is given by $\Beta(\alpha,\beta)$, then upon receiving $n_1$ successes and $n_2$ failures, a Bayesian update yields a posterior proportional to
\[x^{\alpha-1}(1-x)^{\beta-1}x^{n_2}(1-x)^{n_1}=x^{\alpha+n_2-1}(1-x)^{\beta+n_1-1},\]
so after renormalizing, it must be the $\Beta(\alpha+n_2,\beta+n_1)$ distribution. This is what is meant by calling the family of beta distributions a \emph{conjugate prior}: If the prior is a beta distribution, the posterior will be as well.

Surprisingly, despite substantial focus including the entire \emph{Handbook of Beta Distribution and its Applications} \cite{gupta2004handbook}, the following concentration result does not appear to be known:
\begin{thm}The beta distribution $\Beta(\alpha,\beta)$ is $\dfrac1{4(\alpha+\beta)+2}$-subgaussian.\label{betasub}\end{thm}
\begin{rmk}Numerical data suggests that
\[\tau^2(\Beta(\alpha,\beta))\le\tau^2(\Beta((\alpha+\beta)/2,(\alpha+\beta)/2))=\Var(\Beta((\alpha+\beta)/2,(\alpha+\beta)/2))=\frac1{4(\alpha+\beta+1)}.\]
That is, $\Beta(\alpha,\beta)$ appears to be maximized for fixed $\alpha+\beta$ when $\alpha=\beta$, where it appears to be strictly subgaussian. Since the variance is a lower bound for $\tau^2$ (Proposition \ref{subgaussianproperties}.1), we conclude that Theorem \ref{betasub} is tight up to a factor of $1+o(1)$, seen as a function of $\alpha+\beta$. However, numerical data is very clear that the lower bound is tight:\end{rmk}
\begin{conj}The beta distribution $\Beta(\alpha,\beta)$ is $\dfrac1{4(\alpha+\beta+1)}$-subgaussian.\label{betaconj}\end{conj}
Before proving Theorem \ref{betasub}, let's see what this concentration result means for a nearly trivial case of Bayesian adaptive data analysis.
\begin{cor}If the prior on the parameter of a Bernoulli distribution is given by $\Beta(\alpha,\beta)$ for any $\alpha,\beta$, the posterior mean curator strategy wins.\label{bernoullicor}\end{cor}
\begin{proof}[Proof]If the prior is $\Beta(\alpha,\beta)$, the posterior is $\Beta(\alpha',\beta')$ with $\alpha'+\beta'=\alpha+\beta+n>n$.

With a Bernoulli trial, there are two nontrivial counting queries, the probabilities of success and failure, or $X$ and $1-X$. The latter is an affine function of $X$, so by Propositions \ref{subgaussianproperties}.2 and \ref{countingsufficiency} it suffices to show that $X\sim\Beta(\alpha',\beta')$ is $O(1/n)$-subgaussian.

Indeed, Theorem \ref{betasub} shows that the posterior is $\dfrac1{4(\alpha'+\beta')+2}<\dfrac1{4n}$-subgaussian, so by Proposition \ref{subgaussiansufficiency}, the posterior mean curator strategy wins.\end{proof}

Thus, we see that as a conjugate family, the beta distributions have an interesting property: Only some beta distributions can be posteriors after observing and updating on $n$ data points. All such posteriors have $\alpha+\beta>n$, and Theorem \ref{betasub} says that they must concentrate to the degree we require.

\subsection{Probabilistic Proof}

We will provide two very different proofs of results like Theorem \ref{betasub}. We'll start with the probabilistic result, which actually can only produce a weaker constant:
\begin{prop}The beta distribution $\Beta(\alpha,\beta)$ is $\dfrac1{2(\alpha+\beta+1)}$-subgaussian.\label{weakbetasub}\end{prop}
\begin{proof}Notice that technically, this is a claim about the centered beta distribution, or $X-\E X=X-\frac\alpha{\alpha+\beta}$. However, we do not prove this claim from the centered moments, but from the raw moments themselves. That is, rather than proving
\[\E\left[\exp\left(\lambda\left(X-\frac\alpha{\alpha+\beta}\right)\right)\right]\le\exp\left(\lambda^2\frac1{4(\alpha+\beta+1)}\right)\]
by expanding termwise in $\lambda$, we move the mean term to the other side and show
\begin{equation}\E[\exp(\lambda X)]\le\exp\left(\lambda\frac{\alpha}{\alpha+\beta}+\lambda^2\frac1{4(\alpha+\beta+1)}\right)\label{betabound}\end{equation}
by expanding termwise in $\lambda$. The coefficients of the left side are the raw moments given in \eqref{betamoments}. To bound the terms in this expansion, we will need the following technical lemma:
\begin{lem}For any nonnegative integer $j$ and $\alpha,\beta>0$,
\[\frac{\alpha+j}{\alpha+\beta+j}\cdot\frac{\alpha+j+1}{\alpha+\beta+j+1}\le\left(\frac\alpha{\alpha+\beta}\right)^2+\frac{j+1}{2(\alpha+\beta+1)}.\]\label{betatechnical}\end{lem}
Before proving this lemma, let us see how it implies Theorem \ref{betasub}. First, we consider the even terms:
\begin{align*}
[\lambda^{2k}]\E[\exp(\lambda X)]&=\frac1{(2k)!}\E[X^{2k}]=\frac1{(2k)!}\left(\frac{\alpha}{\alpha+\beta}\cdot\frac{\alpha+1}{\alpha+\beta+1}\right)\dotsm\left(\frac{\alpha+2k-2}{\alpha+\beta+2k-2}\cdot\frac{\alpha+2k-1}{\alpha+\beta+2k-1}\right)\\
&\le\frac1{(2k)!}\left(\left(\frac\alpha{\alpha+\beta}\right)^2+\frac1{2(\alpha+\beta+1)}\right)\dotsm\left(\left(\frac\alpha{\alpha+\beta}\right)^2+\frac{2k-1}{2(\alpha+\beta+1)}\right)\\
&\le\frac1{(2k)!}\sum_{l=0}^k\left(\frac\alpha{\alpha+\beta}\right)^{2l}\left(\frac1{2(\alpha+\beta+1)}\right)^{k-l}\binom kl(2k-1)(2k-3)\dotsm(2l+1)\\
&=\frac1{(2k)!}\sum_{l=0}^k\left(\frac\alpha{\alpha+\beta}\right)^{2l}\left(\frac1{4(\alpha+\beta+1)}\right)^{k-l}\frac{(2k)!}{(2l)!(k-l)!}\\
&=\sum_{l=0}^k\frac1{(2l)!}\left(\frac\alpha{\alpha+\beta}\right)^{2l}\frac1{(k-l)!}\left(\frac1{4(\alpha+\beta+1)}\right)^{k-l}\\
&=[\lambda^{2k}]\exp\left(\frac{\lambda\alpha}{\alpha+\beta}+\frac{\lambda^2}{4(\alpha+\beta+1)}\right).
\end{align*}
In going from the first to the second line, we have replaced the pairs of consecutive terms by applying Lemma \ref{betatechnical}. From the second to the third, we have expanded the product and grouped terms, upper bounding the sum of all products of $k-l$ of the right fractions by $\binom kl$ times the largest of them. From the third to the fourth lines, we have expanded $\binom kl=\frac{k(k-1)\dotsc(l+1)}{(k-l)!}$ and doubled the terms in the numerator to interleave with the odd terms, adding an extra factor of 2 to the fraction raised to the $k-l$ power.

The odd terms are similar, except we pair the terms up in the moment starting from the second.
\begin{align*}
[\lambda^{2k+1}]\E[\exp(\lambda X)]&=\frac1{(2k+1)!}\E[X^{2k+1}]\\
&=\frac1{(2k+1)!}\frac\alpha{\alpha+\beta}\left(\frac{\alpha+1}{\alpha+\beta+1}\cdot\frac{\alpha+2}{\alpha+\beta+2}\right)\dotsm\left(\frac{\alpha+2k-1}{\alpha+\beta+2k-1}\cdot\frac{\alpha+2k}{\alpha+\beta+2k}\right)\\
&\le\frac1{(2k+1)!}\frac\alpha{\alpha+\beta}\left(\left(\frac\alpha{\alpha+\beta}\right)^2+\frac2{2(\alpha+\beta+1)}\right)\dotsm\left(\left(\frac\alpha{\alpha+\beta}\right)^2+\frac{2k}{2(\alpha+\beta+1)}\right)\\
&\le\frac1{(2k+1)!}\sum_{l=0}^k\left(\frac\alpha{\alpha+\beta}\right)^{2l+1}\left(\frac1{2(\alpha+\beta+1)}\right)^{k-l}\binom kl(2k+1)(2k-1)\dotsm(2l+3)\\
&=\frac1{(2k+1)!}\sum_{l=0}^k\left(\frac\alpha{\alpha+\beta}\right)^{2l+1}\left(\frac1{4(\alpha+\beta+1)}\right)^{k-l}\frac{(2k+1)!}{(2l+1)!(k-l)!}\\
&=\sum_{l=0}^k\frac1{(2l+1)!}\left(\frac\alpha{\alpha+\beta}\right)^{2l+1}\frac1{(k-l)!}\left(\frac1{4(\alpha+\beta+1)}\right)^{k-l}\\
&=[\lambda^{2k+1}]\exp\left(\frac{\lambda\alpha}{\alpha+\beta}+\frac{\lambda^2}{4(\alpha+\beta+1)}\right).
\end{align*}
The only other major difference is that between the third and fourth lines, we also replaced the numerators of $2k,2k-2,\dotsc,2$ with the corresponding larger values of $2k+1,2k-1,\dotsc,3$. Therefore, for $\lambda>0$, we have shown that all of the terms of the left side of \eqref{betabound} are bounded by the corresponding terms of the right side.

To prove \eqref{betabound} for $\lambda<0$, we utilize the symmetry of the beta distribution: $\Beta(\beta,\alpha)=1-\Beta(\alpha,\beta)$. Therefore, we've in fact also shown that for $\lambda>0$,
\[\E[e^{\lambda(1-X)}]\le\exp\left(\frac{\lambda\beta}{\alpha+\beta}+\frac{\lambda^2}{4(\alpha+\beta+1)}\right).\]
Dividing both sides by $e^\lambda$, we immediately get
\[\E[e^{-\lambda X}]\le\exp\left(-\frac{\lambda\alpha}{\alpha+\beta}+\frac{\lambda^2}{4(\alpha+\beta+1)}\right),\]
so the desired bound holds for $\lambda<0$ as well. We are done, apart from proving the technical lemma.\end{proof}
\begin{proof}[Proof of Lemma \ref{betatechnical}]We induct on $j$, starting with two base cases: $j=0$ and $j=1$.
\begin{bc}For $j=0$, we have
\begin{align*}
\frac\alpha{\alpha+\beta}\cdot\frac{\alpha+1}{\alpha+\beta+1}-\left(\frac\alpha{\alpha+\beta}\right)^2&=\frac\alpha{\alpha+\beta}\cdot\frac\beta{(\alpha+\beta)(\alpha+\beta+1)}\\
&\le\frac{(\alpha+\beta)^2/4}{(\alpha+\beta)^2(\alpha+\beta+1)}=\frac1{4(\alpha+\beta+1)}<\frac1{2(\alpha+\beta+1)},
\end{align*}
where we used the AM-GM inequality $\sqrt{\alpha\beta}\le(\alpha+\beta)/2$. Now, if $j=1$, we have
\begin{align*}
\frac{\alpha+1}{\alpha+\beta+1}\cdot\frac{\alpha+2}{\alpha+\beta+2}-\left(\frac\alpha{\alpha+\beta}\right)^2&=\frac{(\alpha+1)(\alpha+2)(\alpha+\beta)^2-\alpha^2(\alpha+\beta+1)(\alpha+\beta+2)}{(\alpha+\beta)^2(\alpha+\beta+1)(\alpha+\beta+2)}\\
&=\frac{(\alpha^2+3\alpha+2)(\alpha+\beta)^2-\alpha^2((\alpha+\beta)^2+3(\alpha+\beta)+2)}{(\alpha+\beta)^2(\alpha+\beta+1)(\alpha+\beta+2)}\\
&=\frac{(3\alpha+2)(\alpha+\beta)^2-\alpha^2(3(\alpha+\beta)+2)}{(\alpha+\beta)^2(\alpha+\beta+1)(\alpha+\beta+2)}\\
&=\frac{3\alpha\beta(\alpha+\beta)+4\alpha\beta+2\beta^2}{(\alpha+\beta)^2(\alpha+\beta+1)(\alpha+\beta+2)}\\
&\le\frac{(\alpha+\beta)^3+2(\alpha+\beta)^2}{(\alpha+\beta)^2(\alpha+\beta+1)(\alpha+\beta+2)}=\frac1{\alpha+\beta+1},
\end{align*}
as desired.\end{bc}
\begin{is}For the inductive step, take $j\ge2$. Then
\begin{align*}
\frac{\alpha+j}{\alpha+\beta+j}\cdot\frac{\alpha+j+1}{\alpha+\beta+j+1}-\frac{\alpha+j-1}{\alpha+\beta+j-1}\cdot\frac{\alpha+j}{\alpha+\beta+j}&=\frac{\alpha+j}{\alpha+\beta+j}\cdot\frac{2\beta}{(\alpha+\beta+j-1)(\alpha+\beta+j+1)}\\
&\le\frac{(\alpha+\beta+j)^2/2}{(\alpha+\beta+j-1)(\alpha+\beta+j)(\alpha+\beta+j+1)}\\
&<\frac1{2(\alpha+\beta+j-1)}\le\frac1{2(\alpha+\beta+1)}\\
\frac{\alpha+j}{\alpha+\beta+j}\cdot\frac{\alpha+j+1}{\alpha+\beta+j+1}-\left(\frac\alpha{\alpha+\beta}\right)^2&<\frac1{2(\alpha+\beta+1)}+\frac j{2(\alpha+\beta+1)}=\frac{j+1}{2(\alpha+\beta+1)},
\end{align*}
where we applied the inductive hypothesis to get the last line.\qedhere\end{is}\end{proof}

Unfortunately, it doesn't appear possible to squeeze an additional factor of 2 out of this method. In fact, \eqref{betabound} is not true termwise if we replace the $4(\alpha+\beta+1)$ denominator with $8(\alpha+\beta+1)$.\footnote{For a concrete example, when $\alpha=1$ and $\beta=2$, the $\lambda^4$ terms have coefficients of $1/360>1363/497664$, respectively, the opposite direction as needed to prove \eqref{betabound}. However, the coefficients of $\lambda^2$ and $\lambda^6$ correct for this $\lambda^4$ term, and $\Beta(1,2)$ is still $1/4(1+2+1)=1/16$-subgaussian.} Therefore any proof of Conjecture \ref{betaconj} will have to use a different method.

\subsection{Learning-Theoretic Proof}

Quite surprisingly, we can get a simpler and stronger result by considering the beta distribution as a Bayesian posterior. The key is Azuma's Inequality:
\begin{lem}[Azuma's Inequality (c.f. \cite{van2014probability} Lemma 3.7)]Let $\{X_k\}_{k\in\N}$ be a martingale adapted to the filtration $\{\F_k\}_{k\in\N}$ such that for all $k$, $(X_k-X_{k-1})|\F_{k-1}$ is $\sigma_k^2$-subgaussian. Then $(X_n-X_0)|\F_0$ is $\sum_{k=1}^n\sigma_k^2$-subgaussian.\end{lem}
\begin{proof}[Proof of Theorem \ref{betasub}]The martingale we will construct is surprisingly related to our problem: Let $\F_0$ be the $\Beta(\alpha,\beta)$ prior over the parameter of a Bernoulli random variable, let $\F_k$ be the updated information upon receiving $k$ samples from the random variable, and let $X_k$ be the resulting posterior mean. Then as $n\to\infty$, $X_n$ approaches the true parameter of the random variable almost surely, so $X_n-X_0$ approaches the error of the original posterior mean.

Moreover, it is fairly elementary to check that $X_k-X_{k-1}$ is a martingale with the appropriate subgaussian variance proxy. Suppose that the posterior after the first $k-1$ data points is $\Beta(\alpha',\beta')$, where of course $\alpha'+\beta'=\alpha+\beta+k-1$. Then the posterior mean $X_{k-1}=\dfrac{\alpha'}{\alpha'+\beta'}$. Conditioning on the samples the curator has seen so far, the next sample will be a success with probability $X_{k-1}$ and a failure with probability $1-X_{k-1}$. That is:
\begin{align*}
\text{w.p. }\frac{\alpha'}{\alpha'+\beta'},\quad X_k=\frac{\alpha'+1}{\alpha'+\beta'+1}&\implies X_k-X_{k-1}=\frac{\alpha'+1}{\alpha'+\beta'+1}-\frac{\alpha'}{\alpha'+\beta'}=\frac{\beta'}{(\alpha'+\beta')(\alpha'+\beta'+1)};\\
\text{w.p. }\frac{\beta'}{\alpha'+\beta'},\quad X_k=\frac{\alpha'}{\alpha'+\beta'+1}&\implies X_k-X_{k-1}=\frac{\alpha'}{\alpha'+\beta'+1}-\frac{\alpha'}{\alpha'+\beta'}=\frac{-\alpha'}{(\alpha'+\beta')(\alpha'+\beta'+1)}
\end{align*}
This difference is clearly mean-zero, so $\{X_k\}$ is indeed a martingale. As a centered binary random variable, by Theorem 3.1 in \cite{buldygin2013sub}, the subgaussian variance proxy of $X_k-X_{k-1}$ is given by
\[\tau^2(X_k-X_{k-1})=\frac1{(\alpha'+\beta'+1)^2}K\left(\frac{\alpha'}{\alpha'+\beta'}\right),\text{ where }K(p)=\begin{cases}
0&\:p\in\{0,1\}\\
\frac14&\:p=\frac12\\
\frac{p-(1-p)}{2(\ln p-\ln(1-p))}&\:\text{otherwise}.\end{cases}\]
In particular (Lemma 2.1 in \cite{buldygin2013sub}), $K(p)\le\frac14$ for all $p\in[0,1]$, so $X_k-X_{k-1}$ is $\dfrac1{4(\alpha'+\beta'+1)^2}=\dfrac1{4(\alpha+\beta+k)^2}$-subgaussian.

Therefore, by Azuma's inequality, $\Beta(\alpha,\beta)$ is subgaussian with variance proxy
\begin{align*}
\sum_{k=1}^\infty\frac1{4(\alpha+\beta+k)^2}&\le\sum_{k=1}^\infty\frac1{4(\alpha+\beta+k+1/2)(\alpha+\beta+k-1/2)}\\
&=\sum_{k=1}^\infty\left(\frac1{4(\alpha+\beta+k-1/2)}-\frac1{4(\alpha+\beta+k+1/2)}\right)=\frac1{4(\alpha+\beta)+2},
\end{align*}
as desired.\footnote{The final steps of this analysis are also tight; $\sum_{k=1}^\infty1/(4(\alpha+\beta+k)^2)\ge1/((4(\alpha+\beta)+2+1/(3(\alpha+\beta)))$.}\end{proof}

Impressively, this method is able to prove a strictly stronger result than the probabilistic approach, matching the correct coefficient on $\alpha+\beta$ for the symmetric case.

\section{Dirichlet Priors on Categorical Random Variables}\label{dpcrv}

Of course, the example considered in Corollary \ref{bernoullicor} is a nearly trivial example of Bayesian adaptive data analysis: There is (essentially) only one possible query, making adaptivity meaningless. However, the result fortunately generalizes to a much more useful framework: Dirichlet priors on categorical random variables.

Recall that a categorical random variable has support $\{1,\dotsc,k\}$ for some positive integer $k$, and probabilities $p_1,\dotsc,p_k$ for each value. $k=2$ corresponds to a Bernoulli trial again, but if $k>2$, there are many possible queries, each corresponding to a vector $\vec v=(f(1),\dotsc,f(k))\in[0,1]^k$ and asking for the dot product $\vec p\cdot\vec v$.

The conjugate prior for the categorical random variable is the Dirichlet distribution $\Dir(\alpha_1,\dotsc,\alpha_k)$, the natural generalization of the beta distribution. Its probability density function is given by
\[\frac{\Gamma(\alpha_1+\dotsb+\alpha_k)}{\Gamma(\alpha_1)\dotsm\Gamma(\alpha_k)}x_1^{\alpha_1}\dotsm x_k^{\alpha_k}.\]
Therefore, upon receiving data with counts $c_i$ of category $i$, the posterior is given by $\Dir(\alpha_1+c_1,\dotsc,\alpha_k+c_k)$, making the Dirichlet distribution a conjugate prior for the categorical distribution.

\begin{thm}In the direction of any query vector $\vec v$, the Dirichlet distribution is $\frac1{2(\alpha_1+\dotsc+\alpha_k+1)}$-subgaussian.\label{Dirichletbound}\end{thm}
In exactly the same way, this guarantees accuracy of the posterior mean:
\begin{cor}If the prior on the parameter of a categorical random variable is $\Dir(\alpha_1,\dotsc,\alpha_k)$ for any $\alpha_1,\dotsc,\alpha_k$, the posterior mean curator strategy wins.\label{Dirichletcor}\end{cor}
\begin{proof}[Proof of Theorem \ref{Dirichletbound}]By Proposition \ref{countingsufficiency}, it suffices to check this for counting queries, or $v\in\{0,1\}^k$. We will show that the distribution of the Dirichlet prior with respect to such queries is merely a beta distribution,\footnote{This fact might be well-known, but it also isn't hard to prove. It's at least well-known that the marginals in each category are beta distributions, but those only cover the case where only one $v_i=1$.} and apply Theorem \ref{betasub}.

By relabeling coordinates, we can suppose that $v_1=\dotsb=v_l=1$ and $v_{l+1}=\dotsb=v_k=0$ for some $2\le l\le k-1$ (if all $v_i=0$ or all $v_i=1$ the dot product is always $0$ or $1$ respectively). We first transform the simplex of possible $\vec p$ in a fairly common way by considering the partial sums $s_1=p_1,s_2=p_1+p_2,\dotsc,s_{k-1}=1-p_k$. Then the simplex is given by
\[\Delta_k'=\{s_1,\dotsc,s_{k-1}\in\R:0\le s_1\le s_2\le\dotsb\le s_{k-1}\le 1\}.\]

In this notation, $\vec p\cdot\vec v=t$ corresponds to $p_1+\dotsc+p_l=s_l=t$. The probability density of $\vec p\cdot\vec v$ at $t$ is therefore proportional to the $(k-2)$-dimensional volume with respect to the Dirichlet distribution of this slice, or
\begin{align*}
\Pr[v\cdot p=t]&\propto\int_{\Delta_k'\cap\{s_l=t\}}s_1^{\alpha_1-1}(s_2-s_1)^{\alpha_2-1}\dotsm(1-s_{k-1})^{\alpha_k-1}\,ds_1\dotsm ds_{l-1}ds_{l+1}\dotsm ds_{k-1}\\
&=\int_{0\le s_1\le\dotsb\le s_{l-1}\le t}s_1^{\alpha_1-1}\dotsm(t-s_{l-1})^{\alpha_l-1}\,ds_1\dotsm ds_{l-1}\times\\
&\qquad\int_{t\le s_{l+1}\le\dotsb\le s_{k-1}\le 1}(s_{l+1}-t)^{\alpha_{l+1}-1}\dotsm(1-s_{k-1})^{\alpha_k-1}ds_{l+1}\dotsm ds_{k-1}\\
&=\int_{0\le s_1'\le\dotsb\le s_{l-1}'\le1}(ts_1')^{\alpha_1-1}\dotsm(t-ts_{l-1}')^{\alpha_l-1}\,d(ts_1')\dotsm d(ts_{l-1}')\times\\
&\qquad\int_{0\le s_{l+1}'\le\dotsb\le s_{k-1}'\le1}((1-t)s_{l+1}')^{\alpha_{l+1}-1}\dotsm((1-t)(1-s_k'))^{\alpha_k-1}\,d((1-t)s_{l+1}')\dotsm d((1-t)s_k')\\
&\propto t^{\alpha_1+\dotsb+\alpha_l-1}(1-t)^{\alpha_{l+1}+\dotsb+\alpha_k-1},
\end{align*}
where we have substituted $s_i=ts_i'$ for $i\le l$ and $s_j=(1-t)s_j'+t$ for $j>l$. Pulling out the factors of $t$ and $1-t$, the remainder no longer depends on $t$. After normalizing, this is exactly the $\Beta(\alpha_1+\dotsb+\alpha_l,\alpha_{l+1}+\dotsb+\alpha_k)$ distribution. Therefore, by Theorem \ref{betasub}, this distribution is $\dfrac1{2(\alpha_1+\dotsb+\alpha_k+1)}$-subgaussian, as desired.\end{proof}

\section{Discussion}\label{disc}

These results can be extended in multiple directions. First, we examine the core technique, and then we examine other potential conjugate priors.

\subsection{Subgaussianity from Raw Moments}\label{srm}

The proof technique used in Theorem \ref{betasub} is perhaps the most generally useful contribution of this paper. Most proofs that a random variable is subgaussian involve showing a bound on its centered moments like in Proposition \ref{subgaussianproperties}.3. However, this result only used the \emph{raw} moments. In fact, this is all we need:
\begin{lem}If $X$ is a random variable with positive raw moments satisfying
\begin{equation}
\frac{\E[X^{j+2}]}{\E[X^j]}\le\E[X]^2+(j+1)\sigma^2,
\end{equation}
for every nonnegative integer $j$, then for all $\lambda>0$, $\E[e^{\lambda(X-\E[X])}]\le e^{\lambda^2\sigma^2/2}$.\label{momentbound}\end{lem}
\begin{rmk}Note that for $j=0$, this condition says that $\Var[X]\le\sigma^2$.\end{rmk}
\begin{proof}We simply mimic the proof of Theorem \ref{betasub}. Since all moments are positive, we can apply the assumed inequality to each term of a telescoping product:
\begin{align*}
[\lambda^{2k}]\E[\exp(\lambda X)]&=\frac1{(2k)!}\E[X^{2k}]=\frac1{(2k)!}\frac{\E[X^2]}{\E[X^0]}\cdot\frac{\E[X^4]}{\E[X^2]}\dotsm\frac{\E[X^{2k}]}{\E[X^{2k-2}]}\\
&\le\frac1{(2k)!}(\E[X]^2+\sigma^2)(\E[X]^2+3\sigma^2)\dotsm(\E[X]^2+(2k-1)\sigma^2)\\
&\le\frac1{(2k)!}\sum_{l=0}^k\E[X]^{2l}\sigma^{2(k-l)}\binom kl(2k-1)(2k-3)\dotsm(2l+1)\\
&=\frac1{(2k)!}\sum_{l=0}^k\E[X]^{2l}(\sigma^2/2)^{k-l}\frac{(2k)!}{(2l)!(k-l)!}\\
&=\sum_{l=0}^k\frac1{(2l)!}\E[X]^{2l}\frac1{(k-l)!}(\sigma^2/2)^{k-l}\\
&=[\lambda^{2k}]\exp(\lambda\E[X]+\lambda^2\sigma^2/2).
\end{align*}
The odd terms are also the same:
\begin{align*}
[\lambda^{2k+1}]\E[\exp(\lambda X)]&=\frac1{(2k+1)!}\E[X^{2k+1}]=\frac1{(2k+1)!}\E[X]\frac{\E[X^3]}{\E[X]}\dotsm\frac{\E[X^{2k+1}]}{\E[X^{2k-1}]}\\
&\le\frac1{(2k+1)!}\E[X](\E[X]^2+2\sigma^2)\dotsm((\E[X]^2+2k\sigma^2)\\
&\le\frac1{(2k+1)!}\sum_{l=0}^k\E[X]^{2l+1}\sigma^{2(k-l)}\binom kl(2k+1)(2k-1)\dotsm(2l+3)\\
&=\frac1{(2k+1)!}\sum_{l=0}^k\E[X]^{2l+1}(\sigma^2/2)^{k-l}\frac{(2k+1)!}{(2l+1)!(k-l)!}\\
&=\sum_{l=0}^k\frac1{(2l+1)!}\E[X]^{2l+1}\frac1{(k-l)!}(\sigma^2/2)^{k-l}\\
&=[\lambda^{2k+1}]\exp(\lambda\E[X]+\lambda^2\sigma^2/2).
\end{align*}
Summing all of the terms, which are positive since $\lambda>0$, therefore yields the desired inequality.\end{proof}
For ease of discussion, this definition will be helpful:
\begin{defn}A random variable $X$ is $\sigma^2$-\emph{upper subgaussian} if for all $\lambda>0$, $\E[e^{\lambda X}]\le e^{\lambda^2\sigma^2/2}$. Similarly, $X$ is $\sigma^2$-\emph{lower subgaussian} if this bound holds for all $\lambda<0$.\end{defn}
\begin{prop}If $X$ is $\sigma^2$-upper (resp. lower) subgaussian, then $\Pr[X>\eps]\le e^{-\eps^2/(2\sigma^2)}$ (resp. $\Pr[X<-\eps]\le e^{-\eps^2/(2\sigma^2)}$) for all $\eps>0$.\end{prop}
\begin{proof}This is the standard application of Markov's inequality:
\[\Pr[X\ge\eps]\le e^{-\lambda\eps}\E[e^{\lambda X}]\le e^{\lambda^2\sigma^2/2-\lambda\eps},\]
for any $\lambda>0$. Setting $\lambda=\eps/\sigma^2$ yields the desired result. The lower subgaussian bound is identical.\end{proof}
That these upper bounds on the moments of $X$ only imply bounds on the upper tail of the distribution of $X$ is perhaps not surprising, since large moments distinguish the upper tail much more than they do the lower tail.

We can see that this method works well whenever we can bound ratios of \emph{raw moments} of an uncentered random variable in this fashion. For a second example, consider the Chi distribution, a distribution on $[0,\infty)$, which measures the Euclidean norm of a standard $k$-dimensional Gaussian. Its moments are given by (see \cite{weisstein2003chi})
\[\E[X^j]=2^{j/2}\frac{\Gamma((k+j)/2)}{\Gamma(k/2)}\implies\E[X]=\sqrt2\frac{\Gamma((k+1)/2)}{\Gamma(k/2)};\quad\E[X^2]=k;\quad\E[X^{j+2}]=(k+j)\E[X^j].\]
\begin{prop}The Chi distribution $\chi_k$ is $1$-upper subgaussian.\end{prop}
\begin{proof}To apply Lemma \ref{momentbound}, we already have $\E[X^2]=k$ and $\frac{\E[X^{j+2}]}{\E[X^j]}=k+j=(k-1)+(j+1)(1)$, so it only remains to show that $\E[X]^2\ge k-1$. Because the Gamma function is log-convex, $\frac{\Gamma((k+1)/2)}{\Gamma(k/2)}>\frac{\Gamma(k/2)}{\Gamma((k-1)/2)}$, so
\[\E[X]^2=2\frac{\Gamma((k+1)/2)^2}{\Gamma(k/2)^2}>2\frac{\Gamma((k+1)/2)}{\Gamma(k/2)}\cdot\frac{\Gamma(k/2)}{\Gamma((k-1)/2)}=2\frac{\Gamma((k+1)/2)}{\Gamma((k-1)/2)}=k-1,\]
as desired.\end{proof}
Of course, this is not particularly surprising, given that the probability density function of the $\chi$ distribution is proportional to $x^{k-1}e^{-x^2/2}$. Still, this example illustrates the close relationship between bounds on such ratios of moments and having a subgaussian upper tail.

\subsection{Related Problems in Bayesian Data Analysis}\label{rpbda}

The family of beta distributions have a key monotone property as the posteriors for Bernoulli trials that made this approach possible: When the posterior updates on an additional data point, the positive parameter $\alpha+\beta$ increases by 1. This monotonicity means that only some beta distributions in the family can be the posterior of a curator who has already received $n$ data points, those with $\alpha+\beta>n$.

Other conjugate priors also feature similar monotonicity features:
\begin{itemize}
\item For the $m$-binomial distribution with conjugate prior $\Beta(\alpha,\beta)$, $\alpha+\beta$ increases by $m$ with each binomial sample.
\item For the $m$-multinomial distribution with conjugate prior $\Dir(\alpha_1,\dotsc,\alpha_k)$, $\alpha_1+\dotsc+\alpha_k$ increases with $m$ with each multinomial sample.
\item For the Poisson distribution with conjugate prior $\Gamma(\alpha,\beta)$, the parameter $\beta$ increases by 1 with each Poisson sample.
\item For the geometric distribution with conjugate prior $\Beta(\alpha,\beta)$, $\alpha$ increases by $m$ with each geometric sample.
\end{itemize}

Therefore, by this general approach in Proposition \ref{subgaussiansufficiency}, in any of these conjugate prior models for adaptive data analysis, the posterior mean curator is accurate when $n=O\left(\frac1{\eps^2}\log\frac1\delta\right)$ if each of the projections of these distributions onto queries has the appropriate subgaussian concentration. We formulate the necessary conditions into a series of conjectures (using Proposition \ref{countingsufficiency}):
\begin{conj}For any $S\subset\{0,1,\dotsc,m\}$, the random variable
\[\sum_{k\in S}\binom mkp^k(1-p)^{m-k},\quad p\sim\Beta(\alpha,\beta)\]
is $O\left(\frac m{\alpha+\beta}\right)$-subgaussian.\label{betabinomialconj}\end{conj}
\begin{conj}For any subset $S\subset\{0,1,\dotsc\}$, the random variable
\[\sum_{k\in S}p(1-p)^k,\quad p\sim\Beta(\alpha,\beta)\]
is $O\left(\frac1\alpha\right)$-subgaussian.\label{geometricconj}\end{conj}
\begin{conj}For any subset $S\subset\{(x_1,\dotsc,x_k)\in\Z^n:x_i\ge0,x_1+\dotsb+x_k=m\}$, the random variable
\[\sum_{(x_1,\dotsc,x_k)\in S}\frac{m!}{x_1!\dotsm x_k!}p_1^{x_1}\dotsm p_k^{x_k},\quad(p_1,\dotsc,p_k)\sim\Dir(\alpha_1,\dotsc,\alpha_k)\]
is $O\left(\frac m{\alpha_1+\dotsm+\alpha_k}\right)$-subgaussian.\label{multinomialconj}\end{conj}
\begin{conj}For any subset $S\subset\{0,1,\dotsc\}$, the random variable
\[\sum_{k\in S}\frac{\lambda^ke^{-\lambda}}{k!},\quad\lambda\sim\Gamma(\alpha,\beta)\]
is $O\left(\frac1\beta\right)$-subgaussian.\label{poissonconj}\end{conj}
Numerical data suggests that all of these conjectures are true, and the simplest special cases of each are relatively easy to check using Lemma \ref{momentbound}. However, both the raw moments and the step sizes of most of these transformed random variables are not simple to compute, so new methods will be required to prove any of these conjectures.

\subsection{Implications of the Learning-Theoretic Proof}\label{opbda}

Our learning-theoretic proof of Theorem \ref{betasub} is surprisingly general:
\begin{lem}Under any of the following conditions, the curator's posterior distribution with respect to a query function $f$ if $O(1/n)$-subgaussian if any of these conditions hold:
\begin{enumerate}
\item For all $n$, upon receiving the $n$th sample, the change in the posterior mean with respect to $f$ is $O(1/n^2)$-subgaussian.
\item For all $n$, upon receiving the $n$th sample, the change in the posterior mean with respect to $f$ is almost surely bounded by $O(1/n)$.
\item For all $n$, replacing one of the $n$ samples changes the posterior mean with respect to $f$ by at most $O(1/n)$ almost surely.
\item The posterior mean with respect to $f$ is a $O(1)$-Lipschitz function of the empirical mean with respect to $f$.
\end{enumerate}\label{pmchange}
\end{lem}
\begin{proof}We will first show that the first condition suffices and then show that each condition implies the previous one.
\begin{enumerate}
\item In general, the evolution of the posterior mean with respect to $f$ is a martingale which converges almost surely to the true value. Therefore, by Azuma's inequality, the error in the posterior mean after $n$ data points is $\sum_{k=1}^\infty O\left(\frac1{(n+k)^2}\right)=O\left(\frac1n\right)$-subgaussian, as desired.
\item By Hoeffding's inequality, a $O(1/n)$-bounded random variable is $O(1/n^2)$-subgaussian, so this condition implies the first.
\item This condition implies that upon receiving the $n$th sample, the curator's posterior mean with respect to $f$ falls within an interval of width $O(1/n)$. Since the posterior mean before receiving this sample is the average of the new posterior means (weighted by probability), it also falls into this interval and the previous condition holds.
\item Since $f(x)\in[0,1]$ for all possible samples $x$, replacing the $n$th sample changes the empirical mean by at most $1/n$. So if the posterior mean is an $O(1)$-Lipschitz function of the empirical mean, it will change by at most $O(1/n)$ upon replacing a sample, as desired.\qedhere
\end{enumerate}\end{proof}

By Proposition \ref{subgaussiansufficiency}, this implies that if any of these conditions hold for every possible query function $f:\X\to[0,1]$ (or just counting queries, $f:\X\to\{0,1\}$, by Proposition \ref{countingsufficiency}), the posterior mean curator strategy achieves the static bound of $n=O\left(\frac1{\eps^2}\log\frac q\delta\right)$.

In this light, Theorem \ref{betasub} is a simple consequence of the behavior of the posterior mean as the curator receives more samples. If this behavior is stable (in any of these senses), then it will be accurate.

These results pair nicely with the negative results of \cite{elder2016challenges}. The first key construction in that work was a problem based on error-correcting codes that caused the posterior mean to fluctuate wildly, shifting by a constant with respect to some query even after the curator had seen many data points. These results conversely show that all difficulties in Bayesian adaptive data analysis must include this type of instability to some degree.

\section{Conclusion}\label{conc}

This paper's contributions have two major highlights. First, we showed in Corollary \ref{Dirichletcor} that Bayesian adaptive data analysis with a Dirichlet prior avoids the problem of adaptive overfitting by nature: The posterior mean curator algorithm can in fact answer exponentially many adaptive queries in the number of data points under this model, matching the bound for static queries. These tight and dimension-free results, while specific to the Dirichlet prior, are much stronger than any previous adaptive data analysis algorithms could achieve.

The second major contribution, a key ingredient of the first, is a new probabilisitic result on the concentration of the beta distribution, Theorem \ref{betasub}, which has actually already had implications outside of adaptive data analysis \cite{perry2016statistical}. Given the ubiquity of the beta distribution, we expect many other researchers will be interested in this concentration result.

Both of these contributions featured new core techniques that will also likely be generally useful. In Bayesian adaptive data analysis, we established in Proposition \ref{subgaussiansufficiency} and Lemma \ref{pmchange} simple conditions to check in order to establish that the posterior mean curator is accurate with the static sample complexity. Similarly, we established in Lemma \ref{momentbound} a general technique for turning bounds on the \emph{raw} moments of a distribution into subgaussianity results that could have further applications beyond the beta distribution.

These results also both offer interesting open problems. Empirically, it seems that any appropriate conjugate prior has the same sort of concentration as the Dirichlet prior, suggesting that the posterior mean curator is accurate on an even wider swath of models, which would be encouraging to prove. On the probabilistic side, while Theorem \ref{betasub} is almost tight for the symmetric case of $\alpha=\beta$, it might be improved with similar techniques in the asymmetric case, which could be useful in some applications.

\section{Acknowledgements}

The author would like to thank Jon Kelner and the MIT machine learning theory group for many fruitful conversations.

\bibliography{subgaussiansources}
\bibliographystyle{plain}

\end{document}